\documentclass{article}

\makeatletter
\IfFileExists{neurips_2026.sty}{
  \usepackage[preprint]{neurips_2026}
}{
  \usepackage[margin=1in]{geometry}
  \usepackage{times}
  \renewcommand{\@maketitle}{
    \begin{center}
      {\LARGE\bfseries\@title\par}\vskip 0.8em
      {\large\@author\par}\vskip 0.8em
    \end{center}\par\vskip 0.8em}
}
\makeatother

\usepackage[utf8]{inputenc}
\usepackage[T1]{fontenc}
\usepackage{amsmath,amssymb,amsfonts,amsthm,bm}
\usepackage{booktabs,array,tabularx}
\usepackage{xcolor}
\usepackage{multirow}
\usepackage{url}
\usepackage{natbib}
\usepackage[hidelinks]{hyperref}

\usepackage{amsmath,amssymb,amsthm,mathtools}
\usepackage{bm}
\usepackage[utf8]{inputenc} 
\usepackage[T1]{fontenc}    
\usepackage{hyperref}       
\usepackage{url}            
\usepackage{booktabs}       
\usepackage{amsfonts}       
\usepackage{nicefrac}       
\usepackage{microtype}      
\usepackage{xcolor}         
\usepackage{multirow}
\usepackage{fancybox}
\usepackage{enumitem}
\usepackage{algorithm}
\usepackage{algpseudocode}
\usepackage{caption}
\usepackage{subcaption}
\newcommand{\cref}[1]{\ref{#1}}
\newcommand{\Cref}[1]{\ref{#1}}

\definecolor{MissingRed}{RGB}{150,30,30}

\newcommand{\R}{\mathbb{R}}
\newcommand{\diag}{\operatorname{diag}}
\newcommand{\vecop}{\operatorname{vec}}
\newcommand{\mat}[1]{\mathbf{#1}}
\newcommand{\norm}[1]{\left\lVert#1\right\rVert}
\newcommand{\CRPS}{\operatorname{CRPS}}

\theoremstyle{plain}
\newtheorem{proposition}{Proposition}

\title{Teger: Spatiotemporal Covariance for Probabilistic Traffic Forecasting}

\author{%
  Seyed Mohamad Moghadas$^{1,2}$, Esther Rodrigo Bonet$^{1,2}$, Bruno Cornelis$^{1}$, Adrian Munteanu$^{1,2}\thanks{Senior IEEE Member}$\\
    $^{1}$ ETRO Department, Vrije Universiteit Brussel, B-1050 Brussels, Belgium \\
  $^{2}$ imec Kapeldreef 75, B-3001 Leuven, Belgium\\
  \texttt{Seyed.Mohamad.Moghadas, Esther.Rodrigo.Bonet,Adrian.Munteanu@vub.be} \\
}

\begin{document}
\maketitle

\begin{abstract}
Traffic conditions drift -- demand patterns, incident dynamics, and sensor behavior
shift over a deployment's lifetime -- so a joint uncertainty estimate fit once at
training time and left static will miscalibrate as conditions change. We present
TEGER, a residual covariance model that keeps a forecaster's joint predictive
uncertainty current at test time through closed-form updates, not retraining. A
fixed sensor graph supplies a low-dimensional spatial precision factor encoding
which sensors' errors move together; at inference, Gaussian conditioning corrects
the next forecast's mean and covariance from only the most recently observed
residuals, and an exponential moving-average volatility term rescales marginal
uncertainty to track local drift while preserving the learned correlation structure.
Neither update touches the forecasting backbone's weights, so the same mechanism
attaches to a frozen time-series foundation model: no gradient passes through it,
turning a static point forecast into one with continually refreshed, correlated
uncertainty at negligible added cost. In four traffic datasets, the fixed-graph
covariance improves 60-minute $CRPS_{sum}$ over a temporal-only baseline in different configurations of the backbone and datasets. Applied post hoc to a frozen time-series foundation model, Chronos, the same test-time correction reduces $CRPS_{sum}$ from 0.1798 to 0.1736 without foundation-model
fine-tuning. These results support closed-form test-time residual correction,
rather than a covariance fixed once at training time.~\href{https://anonymous.4open.science/r/Teger-D6F1/}{Anonymous Repo}
\end{abstract}

\section{Introduction}
\label{sec:introduction}

\begin{figure}[t]
\centering
\includegraphics[width=0.8\linewidth]{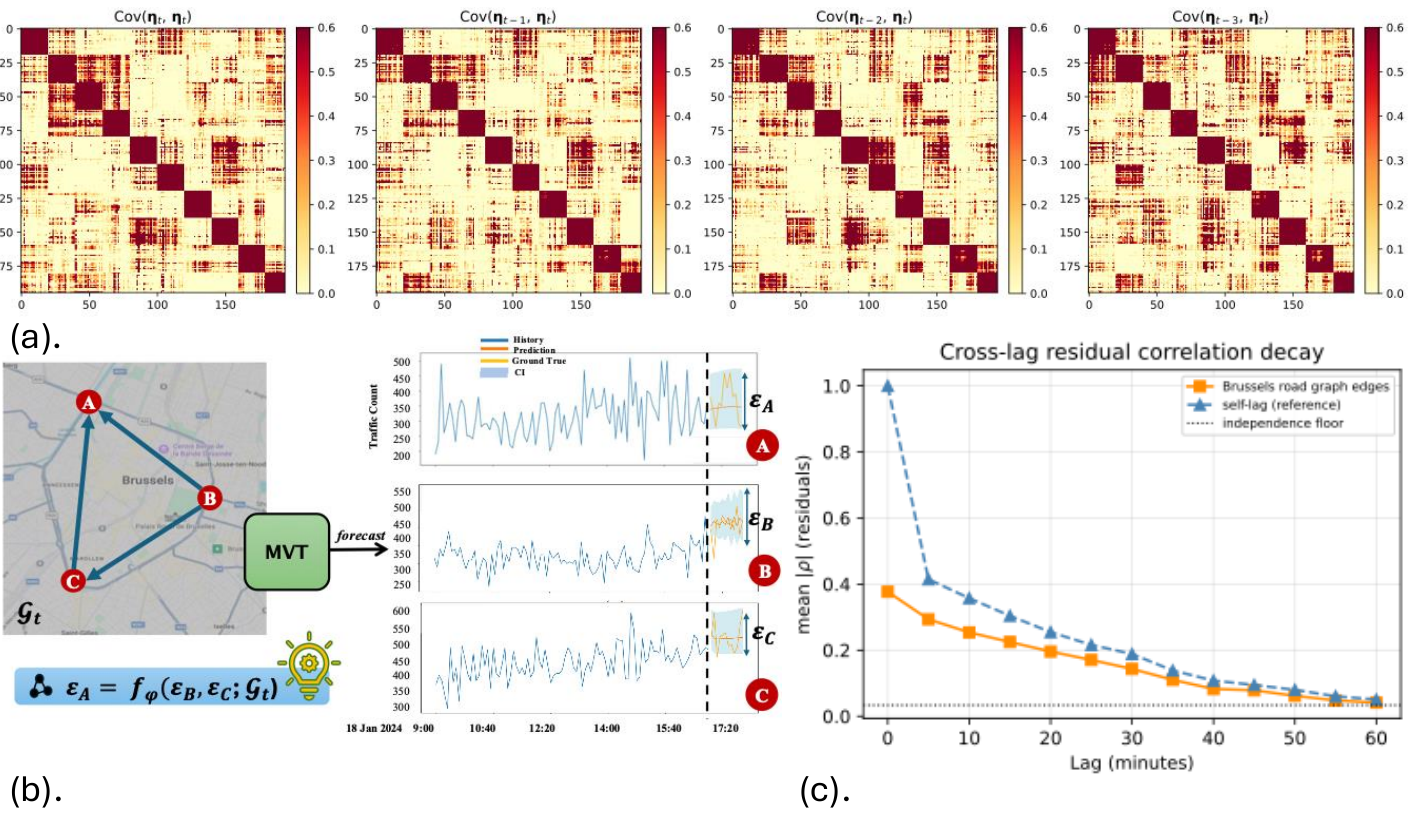}
\caption[Spatiotemporal structure of traffic forecast residuals]{\textbf{Traffic forecast residuals exhibit persistent spatiotemporal dependence.} \textbf{(a)} Residual cross-covariance remains structured across temporal lags. \textbf{(b)} A sensor's residual depends on neighboring residuals. \textbf{(c)} Edge-wise residual correlation decays gradually and varies with network position, motivating graph-aware residual covariance and test-time adaptation.}
\label{fig:crosslag}
\end{figure}

Traffic forecasting supports routing, signal control, and incident response, and
these decisions depend on the joint distribution of forecast errors across sensors
and time, not only on marginal uncertainty at each sensor. Figure~1 shows why this
matters in practice on the Brussels network: panel (a) plots the residual
covariance $\mathrm{Cov}(\eta_t, \eta_{t-k})$ at increasing lags $k$, and the
block structure does not wash out as $k$ grows -- errors at neighboring sensors
keep moving together several steps after the fact. Panel (c) quantifies this
directly: mean absolute residual correlation along road-graph edges decays only
slowly, staying above the independence floor to roughly 40 minutes, and the decay
rate itself depends on where an edge sits in the network. Panel (b) illustrates the mechanism
schematically: on a small three-sensor subgraph, the residual at node $A$ behaves
as a function of its neighbors' residuals and the local graph,
$\varepsilon_A = f_\phi(\varepsilon_B, \varepsilon_C; \mathcal{G}_t)$, rather than
as an independent error term. Two things follow from this: uncertainty should be
modeled jointly across sensors, and because the strength and reach of that joint
structure is heterogeneous and tied to evolving traffic state, an estimate fit
once at training time and left static will drift out of calibration as conditions
change.

Graph neural forecasters address the first point by encoding spatial structure in
the conditional mean (\cite{li2018diffusion, wu2020mtgnn}), and probabilistic
forecasters increasingly model correlated errors over a temporal window (\cite{salinas2019copula, zheng2024correlated}). Neither addresses the second: once trained,
their uncertainty estimates are fixed, and adapting them to a new regime means
retraining. This is a poor fit for deployment, where the forecasting backbone is
often already trained -- or, increasingly, a frozen foundation model whose weights
the practitioner cannot or does not want to touch.

We present TEGER, a residual covariance model that keeps joint uncertainty current
through closed-form updates at inference time rather than through retraining.

Our contributions are:
\begin{itemize}
    \item empirical evidence (Figure~1) that traffic forecast residuals carry
    joint, temporally persistent, and network-position-dependent structure that a
    per-sensor or static covariance model misses;
    \item a fixed-graph latent precision model that supplies this joint structure
    without learning or storing an additional graph, paired with a temporal kernel
    mixture and a low-rank-plus-diagonal observation model into a single
    Kronecker-structured covariance;
    \item two closed-form test-time mechanisms -- conditional residual refinement
    and an exponential moving-average volatility update -- that adapt this
    covariance to non-stationary conditions without retraining the forecasting
    backbone; and
    \item evaluations showing the resulting covariance improves 60-minute $CRPS_{sum}$ over a temporal-only baseline across trained backbones, that the volatility update contributes a further, separately ablated improvement,
    and that the same test-time correction reduces $CRPS_{sum}$ for a frozen Chronos foundation-model forecast with no fine-tuning, using only the fixed-graph.
\end{itemize}

\section{Related Work}
\label{sec:related}

\paragraph{Traffic and multivariate forecasting.}
Deep traffic forecasters combine temporal encoders with spatial operators.
DCRNN uses diffusion convolution \cite{atwood2016diffusionconvolutionalneuralnetworks} on a road graph
\citep{li2018diffusion}; MTGNN learns graph structure jointly with temporal
convolutions \citep{wu2020mtgnn}. Transformer and recurrent architectures also
provide strong sequence backbones \citep{vaswani2017attention,salinas2019copula},
and xLSTM\cite{alharthi2024xlstmtimelongtermtime} extends recurrent memory mechanisms to modern scaling regimes
\citep{beck2024xlstm}. Our contribution is an output covariance model and is
orthogonal to the choice of mean forecaster.

\paragraph{Correlated probabilistic forecasts.}
Low-rank Gaussian and copula constructions make high-dimensional probabilistic
forecasting tractable \citep{salinas2019copula}. Recent work models correlated
errors across a finite temporal window using latent factors and kernel mixtures
\citep{zheng2024correlated}. We retain this temporal construction and add a
spatial factor derived from a supplied graph. The graph is represented in the
precision, following the Gaussian Markov random-field principle that local
conditional interactions are naturally encoded by a graph Laplacian
\citep{rue2005gmrf}.

\section{Method}
\label{sec:method}

\begin{figure}[t]
\centering
\label{model_arch}
\includegraphics[width=0.8\linewidth]{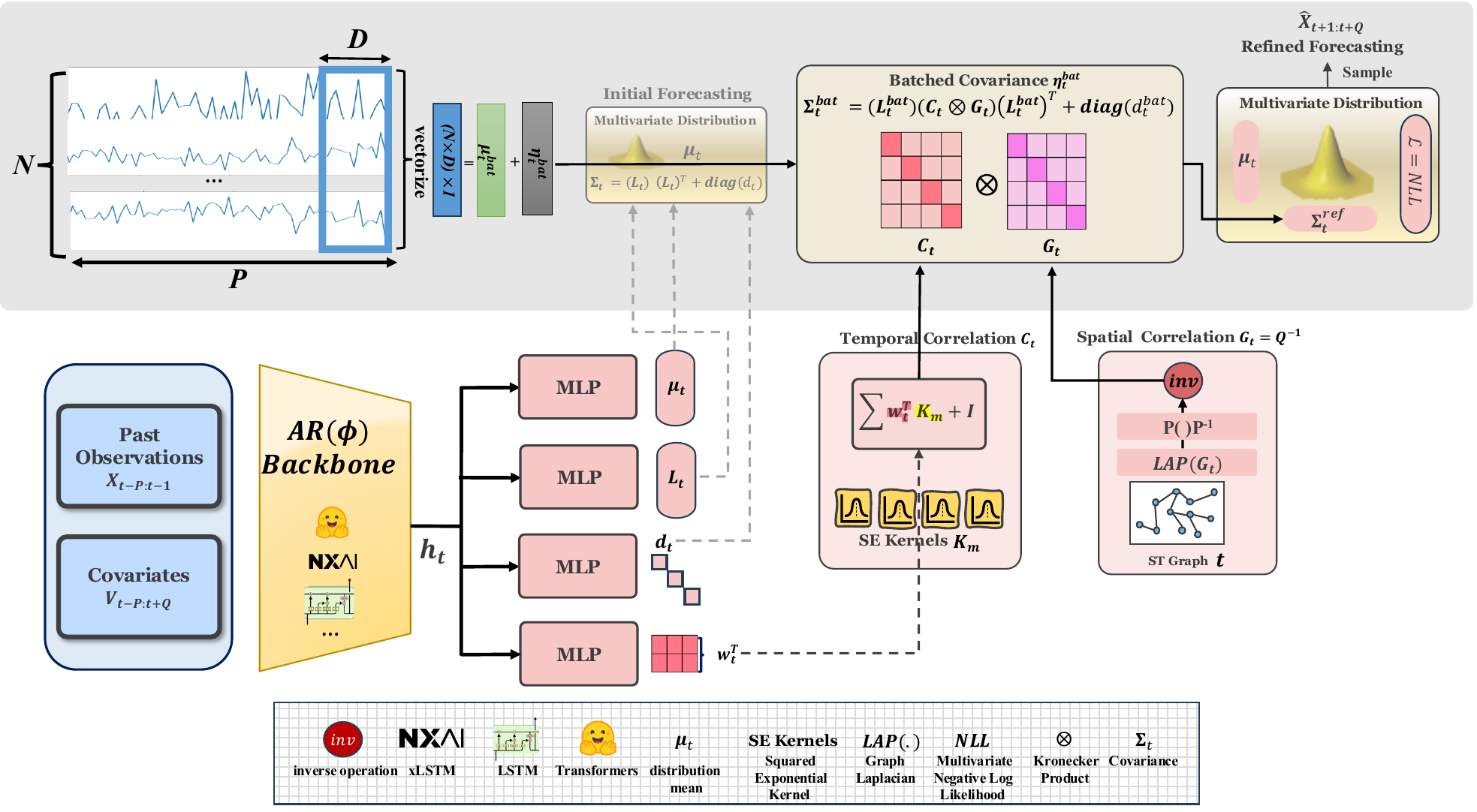}
\caption[Proposed model architecture]{Proposed model architecture}
\end{figure}

Figure~2 lays out the mechanism. A backbone -- any autoregressive architecture,
including a frozen foundation model wrapped as in Section 3.5 -- produces a
hidden state $h_t$ from which lightweight MLP heads read off a conditional mean
$\mu_t$, a low-rank loading $L_t$, and a diagonal variance $d_t$. Separately, and
independently of the backbone, the fixed sensor graph's Laplacian is projected
through a learned map $P$ into a latent precision $Q$, whose inverse $G_t$ supplies
the spatial correlation; a softmax-weighted mixture of squared-exponential kernels
supplies the temporal correlation $C_t$. Their Kronecker product $C_t \otimes G_t$,
combined with the backbone's loadings, forms the batched covariance
$\Sigma_t^{\mathrm{bat}}$ that Figure~2 labels "Initial Forecasting" and that is
fit once, offline, by Gaussian negative log-likelihood. The adaptation happens at
the right-hand side of the figure: at inference, the most recently observed
residuals condition this distribution -- Section 3.4's Gaussian conditioning --
producing the "Refined Forecasting" distribution $\Sigma_t^{\mathrm{ref}}$ without
touching a single backbone weight, and an exponential moving-average volatility
term rescales marginal uncertainty to track local changes in variance while
preserving the learned correlation structure. Because this correction is entirely
a property of the residual stream and never backpropagates into the backbone, it
attaches unchanged to a frozen foundation-model forecast: the point forecast is
whatever the foundation model already produces, and only the joint uncertainty
around it is kept current.

\subsection{Backbone and low-rank residual model}

Let $\mat x_t\in\R^N$ be observations at $N$ sensors. An autoregressive
backbone produces a conditional mean $\mat\mu_t$, a loading matrix
$\mat L_t\in\R^{N\times R}$, and a positive diagonal vector
$\mat d_t\in\R_{>0}^N$:
\begin{equation}
  \mat x_t=\mat\mu_t+\bm\eta_t,\qquad
  \bm\eta_t\sim\mathcal N\!\left(\mathbf{0},
  \mat L_t\mat L_t^\top+\diag(\mat d_t)\right).
  \label{eq:lrd}
\end{equation}
The diagonal component preserves node-specific uncertainty, while the rank-$R$
component provides a tractable channel for shared residual variation.

For a residual window of length $D$, define
$\bm\eta_t^{\mathrm{bat}}=\vecop(\bm\eta_{t-D+1:t})$ and
\begin{equation}
  \mat U_t=\operatorname{blkdiag}(\mat L_{t-D+1},\ldots,\mat L_t).
  \label{eq:block-loadings}
\end{equation}

\subsection{Fixed-graph spatial factor}
\label{sec:spatial}

Let $\mathcal G=(V,E,\mat w)$ be a fixed, undirected, nonnegatively weighted
sensor graph. With an oriented incidence matrix $\mat B$ whose rows correspond
to edges, its Laplacian is
\begin{equation}
  \mat A=\mat B^\top\diag(\mat w)\mat B\succeq\mat 0.
  \label{eq:laplacian}
\end{equation}
Let $\mat P\in\R^{N\times R}$ be a learned column-normalised projection. The
latent precision and covariance are
\begin{equation}
  \mat Q=\gamma\mat I_R+\beta\mat P^\top\mat A\mat P,
  \qquad
  \mat G=\mat Q^{-1},
  \label{eq:spatial-factor}
\end{equation}
where $\gamma>0$ and $\beta\ge0$. The graph support and weights are inputs to
the model and remain fixed throughout training and inference. Learning $\mat P$
adapts how graph variation is represented in the latent space without changing
the graph itself.

\subsection{Temporal mixture and batch covariance}

The temporal factor is a convex mixture of $M$ positive-semidefinite kernels,
\begin{equation}
  \mat C_t=\sum_{m=1}^{M}\pi_{m,t}\mat K_m,
  \qquad \pi_{m,t}\ge0,\quad \sum_m\pi_{m,t}=1.
  \label{eq:temporal}
\end{equation}
Combining temporal and spatial factors gives
\begin{equation}
  \bm\Sigma_t^{\mathrm{bat}}
  =\mat U_t(\mat C_t\otimes\mat G)\mat U_t^\top
   +\diag(\mat d_t^{\mathrm{bat}}).
  \label{eq:batch-cov}
\end{equation}
The model is trained by Gaussian negative log-likelihood. If
$\mat C_t\otimes\mat G=\mat R_t\mat R_t^\top$ and
$\mat V_t=\mat U_t\mat R_t$, Woodbury evaluation uses an inverse of size
$(DR)\times(DR)$. The resulting cost is
$\mathcal O(DNR^2+(DR)^3)$ time and
$\mathcal O(DNR+(DR)^2)$ memory.

\subsection{Conditional residual refinement}
\label{sec:conditioning}

Partition the batch covariance into past residuals $p$ and the next residual
$f$. For observed recent residuals $\mat r_p$, Gaussian conditioning gives
\begin{equation}
  \bm\eta_f\mid\bm\eta_p=\mat r_p
  \sim\mathcal N(\mat M_t\mat r_p,\mat S_t),
  \quad
  \mat M_t=\bm\Sigma_{fp}\bm\Sigma_{pp}^{-1},
  \label{eq:conditional-mean}
\end{equation}
\begin{equation}
  \mat S_t=\bm\Sigma_{ff}
  -\bm\Sigma_{fp}\bm\Sigma_{pp}^{-1}\bm\Sigma_{pf}.
  \label{eq:conditional-cov}
\end{equation}
Thus recent errors can adjust the next mean without retraining the forecasting
backbone.

\paragraph{Volatility adaptation.}
At inference time, a node-wise exponential moving average of recent squared
residuals may replace the conditional marginal scales while preserving the
conditional correlation matrix. The corresponding ablation is reported in Section \cref{sec:ablation}.

\subsection{Post-hoc adaptation of frozen forecasters}
\label{sec:foundation-method}

The covariance model can also be fitted after a pretrained forecaster has been
frozen. Let $\widehat{\mat x}_t^{\mathrm{FM}}$ denote the foundation model's
forecast. We set $\mat\mu_t=\widehat{\mat x}_t^{\mathrm{FM}}$ in
 eq. \cref{eq:lrd} and optimise only the low-rank loadings, diagonal variances,
temporal kernels, fixed-graph spatial factor, and conditional residual model.
No gradient is propagated through the foundation model, and its point forecast
is identical across all residual-model variants. This separates adaptation of
joint uncertainty from fine-tuning of the mean forecaster.

\subsection{Validity and spectral form}

\begin{proposition}[Positive definiteness]
\label{prop:validity}
If $w_e\ge0$ for all edges, $\gamma>0$, $\beta\ge0$,
$\mat C_t\succeq0$, and $\mat d_t^{\mathrm{bat}}$ is elementwise positive,
then $\mat G\succ0$ and $\bm\Sigma_t^{\mathrm{bat}}\succ0$.
\end{proposition}

\begin{proposition}[Graph-resolvent spectrum]
\label{prop:spectrum}
Let $\nu_1,\ldots,\nu_R\ge0$ be the eigenvalues of
$\mat P^\top\mat A\mat P$. The eigenvalues of $\mat G$ are
$(\gamma+\beta\nu_r)^{-1}$. Consequently,
$\lambda_{\max}(\mat G)\le\gamma^{-1}$ and
\begin{equation}
  \operatorname{cond}(\mat G)
  \le \frac{\gamma+\beta\nu_{\max}}{\gamma}.
  \label{eq:condition-bound}
\end{equation}
\end{proposition}
Proofs appear in appendix \cref{app:proofs}.





\begin{table}[t]
\centering
\caption{Dataset summary.}
\label{tab:datasets}
\small
\begin{tabular}{llrr}
\toprule
Dataset & Variable / region & Nodes & Time steps\\
\midrule
PeMS03 & traffic speed / California & 358 & 26,208\\
PeMS04 & traffic speed / California & 307 & 16,992\\
PeMS07 & traffic speed / California & 883 & 26,208\\
Brussels & traffic count / Brussels & 195 & 12,672\\
\bottomrule
\end{tabular}
\end{table}

\begin{table}[htbp]
\centering
\caption{Post-hoc covariance wrapping of frozen time-series foundation
models on Brussels at the 60-minute horizon. No foundation-model weights are
updated in either row of a pair.}
\label{tab:fm-adapt-main}
\begin{tabular}{lcc}
\toprule
\textbf{Variant} & $\textbf{CRPS}_{\textbf{sum}}\downarrow$
& $\textbf{CRPS}_{\textbf{mean}}\downarrow$ \\
\midrule
naive Chronos & $0.1798 \pm 0.0010$ & $0.1755\pm 0.014$ \\
+Teger & $0.1736 \pm 0.0008$ & $0.1729 \pm 0.013$ \\
\midrule
naive Chronos~2 & $0.1679 \pm 0.0012$ & $0.1684 \pm 0.0016$ \\
+Teger & $0.1649 \pm 0.0009$ & $0.1621 \pm 0.0019$ \\
\bottomrule
\end{tabular}
\end{table}

\begin{figure}[!t]
\centering
\includegraphics[width=0.6\linewidth]{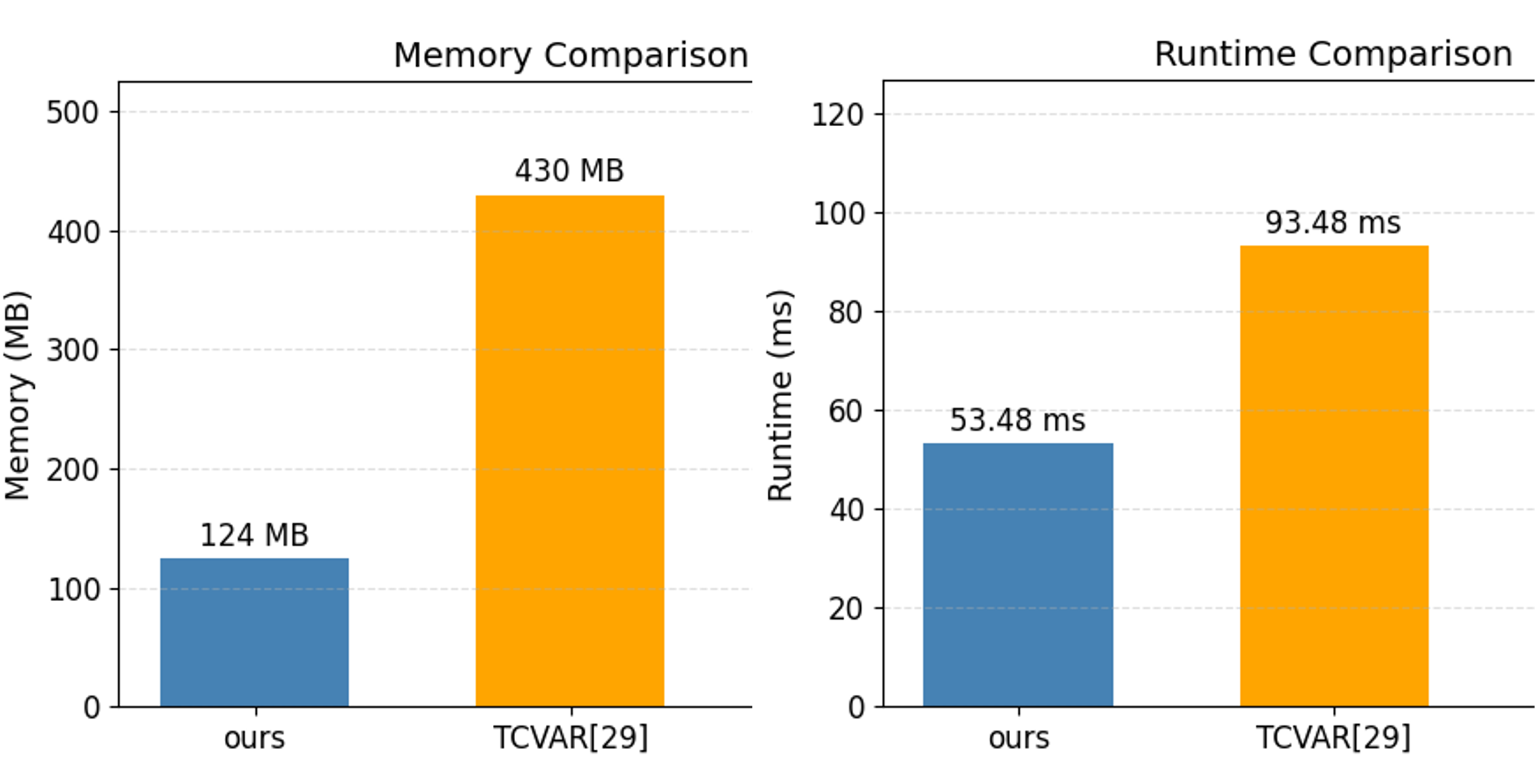}
\caption{Analysis of computation resources for Brussels dataset. Teger consumes resources less than TCVAR~\cite{zheng2024multivariateprobabilistictimeseries} methods, while forecasting more accurately.}
\label{fig3_mem}
\end{figure}



\section{Experiments}
\label{sec:experiments}

\subsection{Experimental Setup}
\label{sec:exp-setup}

We evaluate Teger on four traffic forecasting datasets: PeMS03, PeMS04,
PeMS07, and Brussels. Their statistics are summarized in
\cref{tab:datasets}. The three PeMS benchmarks contain traffic-speed
measurements from California road networks, while Brussels contains traffic
counts from an urban sensor network. We use a chronological
70/10/20\% train/validation/test split for all datasets.

To test whether the proposed residual model is independent of a particular
forecasting architecture, we consider four trainable mean forecasters:
LSTM, Transformer, xLSTM, and MTGNN. The first three provide progressively
different temporal backbones, whereas MTGNN additionally incorporates graph
structure in its mean forecaster. We evaluate forecasting horizons of 15, 30,
and 60 minutes. Following the implementation protocol in
\cref{app:hyperparameters}, the trainable-backbone experiments are evaluated
over five random seeds.

For each backbone, we compare three probabilistic output models:
\begin{enumerate}[leftmargin=*]
    \item \textbf{Diag.}: a diagonal Gaussian output that treats residual
    uncertainty independently across sensors;
    \item \textbf{TCVAR}: the temporally correlated residual model of
    \cite{zheng2024correlated}, which captures dependence across the
    forecasting window but does not use the proposed fixed-graph spatial
    factor; and
    \item \textbf{Teger}: the complete proposed residual model, combining
    low-rank-plus-diagonal uncertainty, the temporal kernel mixture,
    the fixed-graph spatial precision of eq. \cref{eq:spatial-factor},
    conditional residual refinement, and test-time volatility adaptation.
\end{enumerate}
The supplied sensor graph is fixed throughout training and inference: Teger
does not add, remove, reweight, or rewire graph edges. Our primary metric is $\CRPS_{\mathrm{sum}}\downarrow$, i.e., the CRPS of the
aggregate forecast obtained by summing over sensors. Our primary metric is $\CRPS_{\mathrm{sum}}\downarrow$, i.e., the CRPS of
the aggregate forecast obtained by summing over sensors. For the foundation-model experiment,
we additionally report $\CRPS_{\mathrm{mean}}\downarrow$, the average
marginal CRPS across sensors. The full definition of metrics is explained in Section \ref{metrics}.


\subsection{Main Results}
\label{sec:main-results}


Table \Cref{tab:crps-sum-backbones} compares Teger with the diagonal and
temporal-covariance baselines across all four datasets, four backbones, and
three forecasting horizons. Teger improves over TCVAR in 46 of the 48
dataset--backbone--horizon configurations and achieves an average relative
reduction of $11.36\%$ in $\CRPS_{\mathrm{sum}}$ across all configurations.
It also consistently improves over the independent diagonal output,
demonstrating that explicitly modeling residual dependence is important for
aggregate probabilistic forecasting. The gains are particularly clear at the 15- and 30-minute horizons. At
30 minutes, the average reductions relative to TCVAR are $16.51\%$ for LSTM,
$18.72\%$ for Transformer, $19.16\%$ for xLSTM, and $9.94\%$ for MTGNN. The strongest backbone-level result is therefore obtained with xLSTM at
30 minutes, where Teger improves $\CRPS_{\mathrm{sum}}$ by $19.16\%$ on
average across the four datasets. At 15 minutes, the corresponding average
improvements remain substantial at $13.08\%$, $12.28\%$, $13.47\%$, and
$11.18\%$, respectively. At the more challenging 60-minute horizon, Teger outperforms TCVAR in
\emph{all} 16 dataset--backbone combinations. The average relative
improvements are $7.32\%$ for LSTM, $6.56\%$ for Transformer, $6.70\%$ for
xLSTM, and $1.35\%$ for MTGNN. The smaller gain with MTGNN is consistent
with reduced headroom when spatial structure is already incorporated into
the mean forecasting backbone. Importantly, however, the improvement remains
positive for every MTGNN dataset at 60 minutes, indicating that modeling the
joint residual distribution provides information beyond that contained in
the graph-aware point forecaster. Only two configurations favor TCVAR: MTGNN on PeMS04 at 15 minutes
($0.0255$ versus $0.0271$) and MTGNN on PeMS03 at 30 minutes
($0.0422$ versus $0.0453$). Thus, the benefit is not universal for every
individual configuration, but the consistency across the remaining
46 settings and across all 60-minute experiments supports the robustness of
the proposed residual model across datasets, horizons, and backbone classes.

\subsection{Post-hoc Adaptation of Frozen Foundation Models}
\label{sec:foundation-results}

We next test whether Teger can improve predictive uncertainty without
fine-tuning the underlying mean forecaster. We apply the residual model
post hoc to frozen Chronos\cite{ansari2024chronos} and Chronos~2 \cite{ansari2025chronos2univariateuniversalforecasting} forecasts on Brussels at the
60-minute horizon. In all cases, the foundation-model parameters remain
unchanged. Teger is fitted only to the residual process and performs the
closed-form conditional and volatility updates described in
\cref{sec:conditioning}. Thus, any improvement is obtained without
backpropagating through or updating the foundation model.


As shown in Table \cref{tab:fm-adapt-main}, Teger improves both aggregate and
marginal probabilistic accuracy for the two frozen foundation models.
For Chronos, $\CRPS_{\mathrm{sum}}$ decreases from
$0.1798\pm0.0010$ to $0.1736\pm0.0008$, corresponding to a $3.45\%$
relative reduction, while $\CRPS_{\mathrm{mean}}$ decreases from
$0.1755$ to $0.1729$. For Chronos~2, $\CRPS_{\mathrm{sum}}$ decreases
from $0.1679\pm0.0012$ to $0.1649\pm0.0009$, a $1.79\%$ reduction,
while $\CRPS_{\mathrm{mean}}$ improves from $0.1684$ to $0.1621$.

These results are important because the underlying foundation-model
parameters are identical before and after applying Teger. The improvement
therefore comes from modeling and updating the residual distribution rather
than from additional foundation-model training. This demonstrates that the
proposed mechanism can serve as a lightweight probabilistic wrapper around
an already deployed forecaster.

\subsection{Efficiency}
\label{sec:efficiency}

The proposed covariance construction is designed to avoid dense operations on the full $N$-sensor covariance. As discussed in \cref{eq:batch-cov}, the spatial dependence is represented in a rank-$R$ latent space and the Gaussian likelihood is evaluated using the Woodbury
identity, giving complexity
$\mathcal O(DNR^2+(DR)^3)$ rather than requiring direct inversion of a
dense $(DN)\times(DN)$ covariance matrix.

Figure \Cref{fig3_mem} reports the corresponding memory and runtime measurements on
the Brussels dataset. Teger requires less computational memory and runtime
than the temporal covariance baseline in this benchmark while also achieving
lower $\CRPS_{\mathrm{sum}}$. The empirical resource measurements therefore
agree with the low-rank complexity of the proposed construction and show that
the additional spatial dependence and test-time correction do not require
retraining or an expensive full-dimensional covariance update.

%

\subsection{Ablation Study}
\label{sec:ablation}

We isolate two implementation choices at the 60-minute horizon using xLSTM
and MTGNN. To avoid treating different backbones on the same dataset as
independent observations, \cref{tab:aggregate-ablation} first computes the
relative $\CRPS_{\mathrm{sum}}$ change for each backbone and then aggregates
the two backbone results within each dataset. Positive values indicate an
improvement of the first variant over the second.


The volatility update provides a consistent improvement on all four
datasets. Retaining test-time volatility adaptation reduces
$\CRPS_{\mathrm{sum}}$ by $2.47\%$, $4.86\%$, $1.33\%$, and $0.85\%$ on
PeMS03, PeMS04, PeMS07, and Brussels, respectively, corresponding to an
average improvement of $2.38\%$. This result directly supports the role of
the exponential moving-average update in tracking changes in residual scale
without retraining the forecasting backbone.

We additionally evaluate the sensitivity of the spatial factor to Laplacian
normalization. The unnormalized formulation improves over the normalized
form by $0.89\%$, $2.72\%$, $1.15\%$, and $0.71\%$ across the four
datasets, with a mean improvement of $1.37\%$. The comparatively small
magnitude of this effect indicates that the benefit of the residual model is
not driven solely by a particular Laplacian scaling, while the larger
volatility ablation confirms that test-time scale adaptation contributes
independently to the final predictive distribution.

\begin{table}[htbp]
\centering
\caption{Dataset-clustered 60-minute control contrasts. Each dataset entry is
the mean relative $\CRPS_{\mathrm{sum}}$ improvement across xLSTM and MTGNN;
positive values favor the first variant. Backbones are treated as repeated
measurements within a dataset rather than as independent statistical units.}
\label{tab:aggregate-ablation}
\small
\setlength{\tabcolsep}{3.5pt}
\begin{tabular}{lrrrr}
\toprule
Comparison & PeMS03 & PeMS04 & PeMS07 & Brussels \\
\midrule
unnormalized vs.\ normalized Laplacian
& $+0.89\%$ & $+2.72\%$ & $+1.15\%$ & $+0.71\%$ \\
full vs.\ no volatility scale
& $+2.47\%$ & $+4.86\%$ & $+1.33\%$ & $+0.85\%$  \\
graph vs w/o graph & $+13.32\%$ & $+10.81\%$ & $+4.39\%$ & $+7.3\%$ \\
with eq.\ref{eq:conditional-mean} vs w/o conditioning & $+8.12\%$ & $+9.23\%$ & $+5.34\%$ & $+6.32\%$ \\
\bottomrule
\end{tabular}
\end{table}

\section{Limitations}
\label{sec:limitations}

The method assumes that the supplied graph is meaningful. It cannot represent
missing links or learn a new topology, and an inaccurate graph can impose an
unhelpful prior. Because $\mat P$ is learned, graph structure acts in a latent
basis; individual node-space conditional independences are therefore not
directly interpretable. The Gaussian low-rank family cannot express heavy tails
or multimodality.


\section{Conclusion}

We presented Teger, a fixed-graph spatiotemporal residual covariance model that combines graph-structured spatial dependence, temporal kernels, and closed-form test-time adaptation without retraining the forecasting backbone. Across four traffic datasets, four backbones, and three forecasting horizons, Teger improves $\CRPS_{\mathrm{sum}}$ over TCVAR in 46 of 48 configurations, with an average reduction of $11.36\%$. The same mechanism also improves the probabilistic forecasts of frozen Chronos and Chronos~2 models. These results support graph-aware residual modeling and lightweight test-time adaptation as effective tools for improving joint predictive uncertainty.



\bibliographystyle{plain}
\bibliography{biblio}

\appendix

\begin{table*}[h]
\centering
\caption{
$\mathrm{CRPS}_{\mathrm{sum}}\downarrow$ across forecasting horizons,
datasets, and backbone architectures.
Best results within each backbone--horizon setting are shown in
\textbf{bold}, and second-best results are \underline{underlined}.
Avg. Imp. denotes the average percentage reduction in
$\mathrm{CRPS}_{\mathrm{sum}}$ of Teger relative to TCVAR across the
four datasets.
}
\label{tab:crps-sum-backbones}
\resizebox{\textwidth}{!}{
\begin{tabular}{cllccccc}
\toprule
\textbf{Backbone}
& \textbf{Horizon}
& \textbf{Method}
& \textbf{PeMS03}
& \textbf{PeMS04}
& \textbf{PeMS07}
& \textbf{Brussels}
& \textbf{Avg. Imp. (\%)} \\
\midrule

\multirow{9}{*}{LSTM}
& \multirow{3}{*}{15 min}
& Diag.
& 0.0475 $\pm$ 0.0014 & 0.0229 $\pm$ 0.0029 & 0.0968 $\pm$ 0.0041 & 0.1455 $\pm$ 0.0033 & -- \\
&
& TCVAR
& \underline{0.0405 $\pm$ 0.0021}
& \underline{0.0197 $\pm$ 0.0038}
& \underline{0.0956 $\pm$ 0.0017}
& \underline{0.0863 $\pm$ 0.0045}
& -- \\
&
& Teger
& \textbf{0.0356 $\pm$ 0.0032}
& \textbf{0.0172 $\pm$ 0.0026}
& \textbf{0.0826 $\pm$ 0.0011}
& \textbf{0.0742 $\pm$ 0.0049}
& \textbf{13.08} \\
\cmidrule(l){2-8}

& \multirow{3}{*}{30 min}
& Diag.
& 0.0476 $\pm$ 0.0036 & 0.0242 $\pm$ 0.0018 & 0.0972 $\pm$ 0.0022 & 0.1487 $\pm$ 0.0047 & -- \\
&
& TCVAR
& \underline{0.0419 $\pm$ 0.0025}
& \underline{0.0210 $\pm$ 0.0013}
& \underline{0.0949 $\pm$ 0.0039}
& \underline{0.0894 $\pm$ 0.0042}
& -- \\
&
& Teger
& \textbf{0.0395 $\pm$ 0.0019}
& \textbf{0.0141 $\pm$ 0.0027}
& \textbf{0.0804 $\pm$ 0.0031}
& \textbf{0.0786 $\pm$ 0.0015}
& \textbf{16.51} \\
\cmidrule(l){2-8}

& \multirow{3}{*}{60 min}
& Diag.
& 0.0503 $\pm$ 0.0044 & 0.0255 $\pm$ 0.0028 & 0.1001 $\pm$ 0.0016 & 0.1512 $\pm$ 0.0037 & -- \\
&
& TCVAR
& \underline{0.0491 $\pm$ 0.0023}
& \underline{0.0221 $\pm$ 0.0041}
& \underline{0.0989 $\pm$ 0.0029}
& \underline{0.0917 $\pm$ 0.0012}
& -- \\
&
& Teger
& \textbf{0.0476 $\pm$ 0.0035}
& \textbf{0.0176 $\pm$ 0.0048}
& \textbf{0.0957 $\pm$ 0.0024}
& \textbf{0.0893 $\pm$ 0.0017}
& \textbf{7.32} \\

\midrule

\multirow{9}{*}{Transformer}
& \multirow{3}{*}{15 min}
& Diag.
& 0.0457 $\pm$ 0.0021 & 0.0218 $\pm$ 0.0033 & 0.0957 $\pm$ 0.0019 & 0.1668 $\pm$ 0.0045 & -- \\
&
& TCVAR
& \underline{0.0355 $\pm$ 0.0027}
& \underline{0.0156 $\pm$ 0.0014}
& \underline{0.0932 $\pm$ 0.0042}
& \underline{0.0786 $\pm$ 0.0031}
& -- \\
&
& Teger
& \textbf{0.0319 $\pm$ 0.0011}
& \textbf{0.0137 $\pm$ 0.0025}
& \textbf{0.0797 $\pm$ 0.0038}
& \textbf{0.0690 $\pm$ 0.0049}
& \textbf{12.28} \\
\cmidrule(l){2-8}

& \multirow{3}{*}{30 min}
& Diag.
& 0.0473 $\pm$ 0.0022 & 0.0231 $\pm$ 0.0046 & 0.0974 $\pm$ 0.0015 & 0.1671 $\pm$ 0.0034 & -- \\
&
& TCVAR
& \underline{0.0371 $\pm$ 0.0041}
& \underline{0.0171 $\pm$ 0.0029}
& \underline{0.0949 $\pm$ 0.0018}
& \underline{0.0793 $\pm$ 0.0023}
& -- \\
&
& Teger
& \textbf{0.0298 $\pm$ 0.0035}
& \textbf{0.0128 $\pm$ 0.0017}
& \textbf{0.0800 $\pm$ 0.0044}
& \textbf{0.0680 $\pm$ 0.0026}
& \textbf{18.72} \\
\cmidrule(l){2-8}

& \multirow{3}{*}{60 min}
& Diag.
& 0.0490 $\pm$ 0.0013 & 0.0243 $\pm$ 0.0039 & 0.0991 $\pm$ 0.0028 & 0.1717 $\pm$ 0.0012 & -- \\
&
& TCVAR
& \underline{0.0386 $\pm$ 0.0047}
& \underline{0.0178 $\pm$ 0.0021}
& \underline{0.0966 $\pm$ 0.0033}
& \underline{0.0817 $\pm$ 0.0019}
& -- \\
&
& Teger
& \textbf{0.0359 $\pm$ 0.0024}
& \textbf{0.0154 $\pm$ 0.0016}
& \textbf{0.0941 $\pm$ 0.0045}
& \textbf{0.0791 $\pm$ 0.0031}
& \textbf{6.56} \\

\midrule

\multirow{9}{*}{xLSTM}
& \multirow{3}{*}{15 min}
& Diag.
& 0.0448 $\pm$ 0.0018 & 0.0168 $\pm$ 0.0037 & 0.0941 $\pm$ 0.0025 & 0.1448 $\pm$ 0.0042 & -- \\
&
& TCVAR
& \underline{0.0390 $\pm$ 0.0032}
& \underline{0.0155 $\pm$ 0.0015}
& \underline{0.0907 $\pm$ 0.0048}
& \underline{0.0748 $\pm$ 0.0029}
& -- \\
&
& Teger
& \textbf{0.0332 $\pm$ 0.0041}
& \textbf{0.0133 $\pm$ 0.0022}
& \textbf{0.0784 $\pm$ 0.0017}
& \textbf{0.0665 $\pm$ 0.0036}
& \textbf{13.47} \\
\cmidrule(l){2-8}

& \multirow{3}{*}{30 min}
& Diag.
& 0.0464 $\pm$ 0.0019 & 0.0180 $\pm$ 0.0045 & 0.0958 $\pm$ 0.0026 & 0.1473 $\pm$ 0.0011 & -- \\
&
& TCVAR
& \underline{0.0405 $\pm$ 0.0034}
& \underline{0.0166 $\pm$ 0.0028}
& \underline{0.0922 $\pm$ 0.0013}
& \underline{0.0801 $\pm$ 0.0049}
& -- \\
&
& Teger
& \textbf{0.0337 $\pm$ 0.0023}
& \textbf{0.0122 $\pm$ 0.0041}
& \textbf{0.0781 $\pm$ 0.0035}
& \textbf{0.0656 $\pm$ 0.0016}
& \textbf{19.16} \\
\cmidrule(l){2-8}

& \multirow{3}{*}{60 min}
& Diag.
& 0.0479 $\pm$ 0.0047 & 0.0191 $\pm$ 0.0012 & 0.0975 $\pm$ 0.0039 & 0.1496 $\pm$ 0.0024 & -- \\
&
& TCVAR
& \underline{0.0420 $\pm$ 0.0021}
& \underline{0.0177 $\pm$ 0.0033}
& \underline{0.0938 $\pm$ 0.0046}
& \underline{0.0819 $\pm$ 0.0018}
& -- \\
&
& Teger
& \textbf{0.0387 $\pm$ 0.0038}
& \textbf{0.0149 $\pm$ 0.0025}
& \textbf{0.0919 $\pm$ 0.0014}
& \textbf{0.0810 $\pm$ 0.0043}
& \textbf{6.70} \\

\midrule

\multirow{9}{*}{MTGNN}
& \multirow{3}{*}{15 min}
& Diag.
& 0.0539 $\pm$ 0.0015 & 0.0301 $\pm$ 0.0048 & 0.1052 $\pm$ 0.0029 & 0.1471 $\pm$ 0.0037 & -- \\
&
& TCVAR
& \underline{0.0491 $\pm$ 0.0026}
& \textbf{0.0255 $\pm$ 0.0013}
& \underline{0.1013 $\pm$ 0.0041}
& \underline{0.0918 $\pm$ 0.0022}
& -- \\
&
& Teger
& \textbf{0.0371 $\pm$ 0.0044}
& \underline{0.0271 $\pm$ 0.0031}
& \textbf{0.0831 $\pm$ 0.0018}
& \textbf{0.0839 $\pm$ 0.0027}
& \textbf{11.18} \\
\cmidrule(l){2-8}

& \multirow{3}{*}{30 min}
& Diag.
& 0.0509 $\pm$ 0.0019 & 0.0335 $\pm$ 0.0046 & 0.1012 $\pm$ 0.0034 & 0.1507 $\pm$ 0.0015 & -- \\
&
& TCVAR
& \textbf{0.0422 $\pm$ 0.0028}
& \underline{0.0301 $\pm$ 0.0011}
& \underline{0.0955 $\pm$ 0.0049}
& \underline{0.0987 $\pm$ 0.0032}
& -- \\
&
& Teger
& \underline{0.0453 $\pm$ 0.0035}
& \textbf{0.0230 $\pm$ 0.0024}
& \textbf{0.0851 $\pm$ 0.0017}
& \textbf{0.0864 $\pm$ 0.0045}
& \textbf{9.94} \\
\cmidrule(l){2-8}

& \multirow{3}{*}{60 min}
& Diag.
& 0.0583 $\pm$ 0.0041 & 0.0310 $\pm$ 0.0018 & 0.1098 $\pm$ 0.0026 & 0.1535 $\pm$ 0.0039 & -- \\
&
& TCVAR
& \underline{0.0541 $\pm$ 0.0012}
& \underline{0.0275 $\pm$ 0.0043}
& \underline{0.1023 $\pm$ 0.0037}
& \underline{0.0971 $\pm$ 0.0021}
& -- \\
&
& Teger
& \textbf{0.0533 $\pm$ 0.0036}
& \textbf{0.0270 $\pm$ 0.0029}
& \textbf{0.1013 $\pm$ 0.0014}
& \textbf{0.0960 $\pm$ 0.0048}
& \textbf{1.35} \\

\midrule
\multicolumn{7}{r}{\textbf{Overall average improvement of Teger over TCVAR}}
& \textbf{11.36} \\
\bottomrule
\end{tabular}
}
\end{table*}

\section{Proofs}
\label{app:proofs}

\subsection{Proof of \cref{prop:validity}}

For any $\mat x\in\R^N$,
\begin{equation}
  \mat x^\top\mat A\mat x
  =\sum_{e\in E}w_e(\mat b_e^\top\mat x)^2\ge0,
\end{equation}
so $\mat A\succeq0$. Congruence preserves positive semidefiniteness, hence
$\mat P^\top\mat A\mat P\succeq0$. For any nonzero $\mat z\in\R^R$,
\begin{equation}
  \mat z^\top\mat Q\mat z
  =\gamma\norm{\mat z}_2^2
   +\beta\mat z^\top\mat P^\top\mat A\mat P\mat z
  \ge\gamma\norm{\mat z}_2^2>0.
\end{equation}
Therefore $\mat Q\succ0$ and $\mat G=\mat Q^{-1}\succ0$. Since both
$\mat C_t$ and $\mat G$ are positive semidefinite, so is their Kronecker
product. Congruence by $\mat U_t$ preserves positive semidefiniteness, and
adding a strictly positive diagonal yields
$\bm\Sigma_t^{\mathrm{bat}}\succ0$. \hfill$\square$

\subsection{Proof of \cref{prop:spectrum}}

Write
$\mat P^\top\mat A\mat P=\mat V\diag(\nu_r)\mat V^\top$ with orthogonal
$\mat V$ and $\nu_r\ge0$. Then
\begin{equation}
  \mat G
  =\mat V\diag\!\left((\gamma+\beta\nu_r)^{-1}\right)\mat V^\top.
\end{equation}
The largest eigenvalue is at most $\gamma^{-1}$, and the ratio of the largest
to smallest eigenvalue is bounded by
$(\gamma+\beta\nu_{\max})/\gamma$. \hfill$\square$

\section{Metrics}\label{metrics}
To rigorously evaluate probabilistic forecasts across multivariate time series, we utilize the Continuous Ranked Probability Score (CRPS). Standard univariate performance is assessed using $\text{CRPS}_{\text{mean}}$, which computes the expected score across all individual time series $i$ and time steps $t$:
\begin{equation}
\text{CRPS}_{\text{mean}} = \mathbb{E}_{i,t} \left[ \text{CRPS}(F_{i,t}, z_{i,t}) \right],
\end{equation}
where $F_{i,t}$ denotes the predicted empirical cumulative distribution function (CDF) for observation $z_{i,t}$, and individual scores follow the energy statistic formulation $\text{CRPS}(F, z) = \mathbb{E}_F|Z - z| - \frac{1}{2}\mathbb{E}_F|Z - Z'|$ for independent prediction samples $Z, Z' \sim F$. While $\text{CRPS}_{\text{mean}}$ effectively measures the overall calibration and sharpness of isolated marginal distributions, it remains inherently blind to inter-series dependencies, penalizing each dimension independently regardless of cross-sectional correlation.

To capture these critical multivariate dependencies, we complement marginal evaluation with $\text{CRPS}_{\text{sum}}$. This metric evaluates the model's capacity to learn joint dynamics by first aggregating predicted samples and ground-truth values across all dimensions at each time step before computing the loss:
\begin{equation}
\text{CRPS}_{\text{sum}} = \mathbb{E}_{t} \left[ \text{CRPS}\left( F_{t}, \sum_{i} z_{i,t} \right) \right],
\end{equation}
where the aggregate CDF $F_t$ is constructed from the summed prediction samples across all time series $i$. Because the variance of a sum directly depends on the pairwise covariance among its constituents, misestimating cross-series correlations leads to severe variance miscalibration in $F_t$.

\section{Graph Construction}
\label{app:graph}

For the PeMS and Brussels datasets, the adjacency is a thresholded Gaussian
kernel,
\begin{equation}
 A_{ij}^{\mathrm{adj}}=
 \begin{cases}
 \exp(-d_{ij}^2/(2\sigma^2)), &
 \exp(-d_{ij}^2/(2\sigma^2))\ge0.1,\\
 0, & \text{otherwise},
 \end{cases}
\end{equation}
where $d_{ij}$ is road-network distance for PeMS and GPS Euclidean distance
for Brussels, and $\sigma$ is the standard deviation of pairwise distances.
The resulting symmetric connected graph is fixed across the experiments.

\section{Hyperparameters and Implementation}
\label{app:hyperparameters}

\begin{table}[htbp]
\centering
\caption{Hyperparameters retained for the fixed-graph model.}
\label{tab:hyperparameters}
\small
\begin{tabular}{lll}
\toprule
Group & Parameter & Value\\
\midrule
Training & epochs / maximum steps & 100 / 10,000\\
Training & node batch size / gradient clip & 20 / 10.0\\
Training & weight decay & $10^{-8}$\\
Backbone & hidden dimension / layers / dropout & 40 / 2 / 0.01\\
Temporal & window $D$ / kernels $M$ / rank $R$ & 12 / 4 / 10\\
Numerics & diagonal floor & $10^{-4}$\\
Optimisation & learning-rate candidates & $10^{-3},10^{-2}$\\
Evaluation & seeds / rolling windows & 5 / 24\\
\bottomrule
\end{tabular}
\end{table}

Our implementation is done by PyTorch 1.13.1 with CUDA 11.7, an Intel 13900F CPU, and two
NVIDIA RTX 4090 GPUs with 24 GB memory each. The scalars $\beta$ and $\gamma$ are learned.

\section{Notation}
\label{app:notation}

The list of notations are described in Table \ref{notation}.

\begin{table}[htbp]
\centering
\caption{Key notation.}
\small
\label{notation}
\begin{tabularx}{\linewidth}{lX}
\toprule
Symbol & Definition\\
\midrule
$N,D,R,M$ & number of nodes, window length, latent rank, and temporal kernels\\
$\mathcal G=(V,E,\mat w)$ & fixed weighted sensor graph\\
$\mat A$ & fixed graph Laplacian\\
$\mat P$ & learned column-normalised node-to-factor projection\\
$\mat Q,\mat G$ & latent spatial precision and covariance\\
$\mat C_t$ & temporal kernel-mixture correlation matrix\\
$\mat L_t,\mat d_t$ & low-rank loading and positive diagonal variance\\
$\mat U_t$ & block-diagonal loadings over a window\\
$\bm\Sigma_t^{\mathrm{bat}}$ & joint residual covariance over a window\\
$\mat M_t,\mat S_t$ & conditional residual map and covariance\\
\bottomrule
\end{tabularx}
\end{table}

\section{Visualizations}

A qualitative comparison of TEGER against TCVAR and the impact of the volatility block are shown in Figure~\cref{fig:qualitative_and_curvature}.

\begin{figure}[h]
    \centering
    \begin{subfigure}{0.32\textwidth}\includegraphics[width=\textwidth]{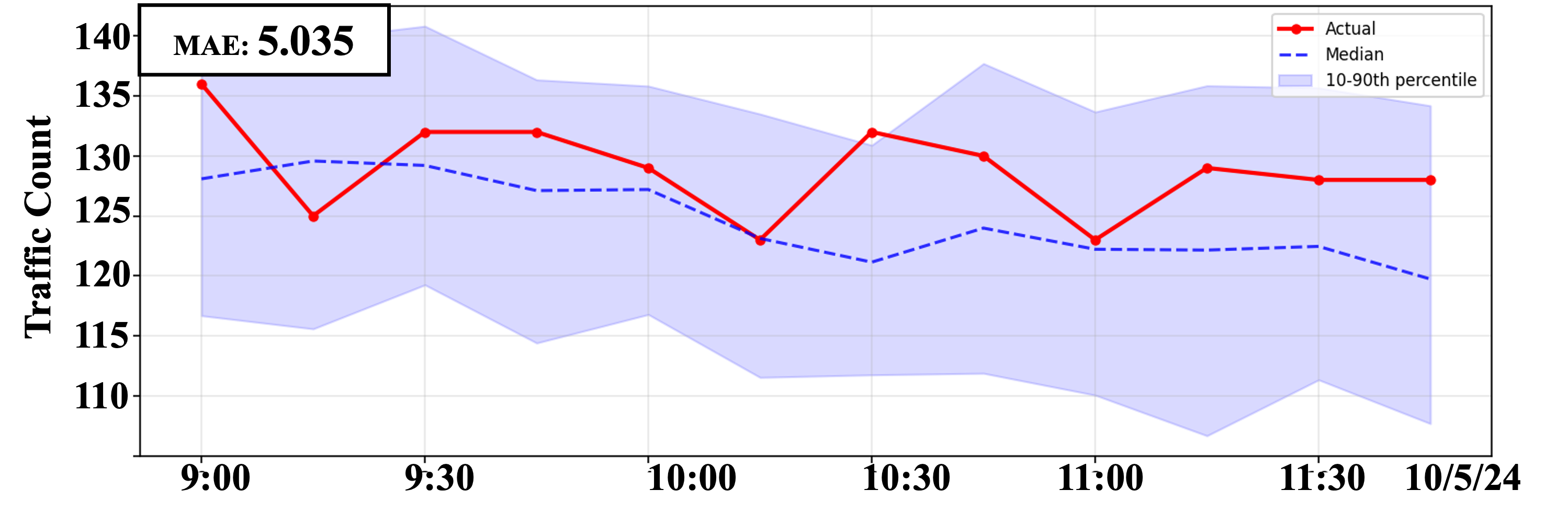}
        \caption{\small TCVAR \cite{zheng2024multivariateprobabilistictimeseries}}
        \label{fig4_a}
        \end{subfigure}
            \begin{subfigure}{0.32\textwidth}\includegraphics[width=\textwidth]{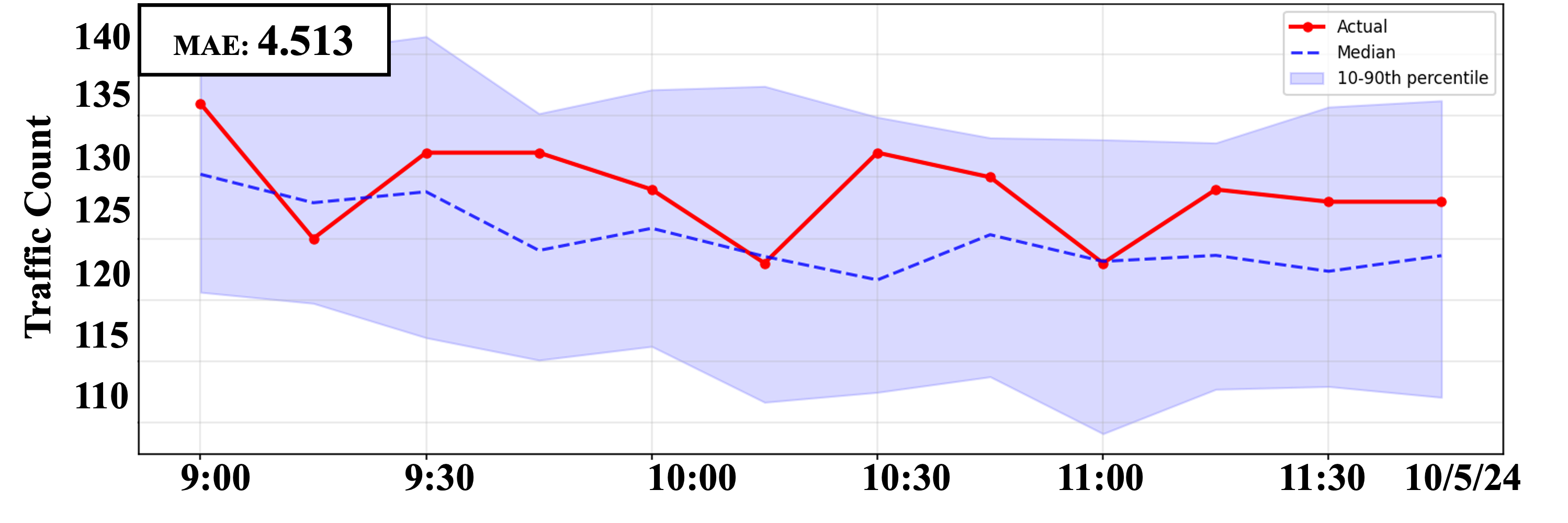}
        \caption{\small Teger w/o volatility}
        \label{fig4_b}
        \end{subfigure}
\begin{subfigure}{0.32\textwidth}\includegraphics[width=\textwidth]{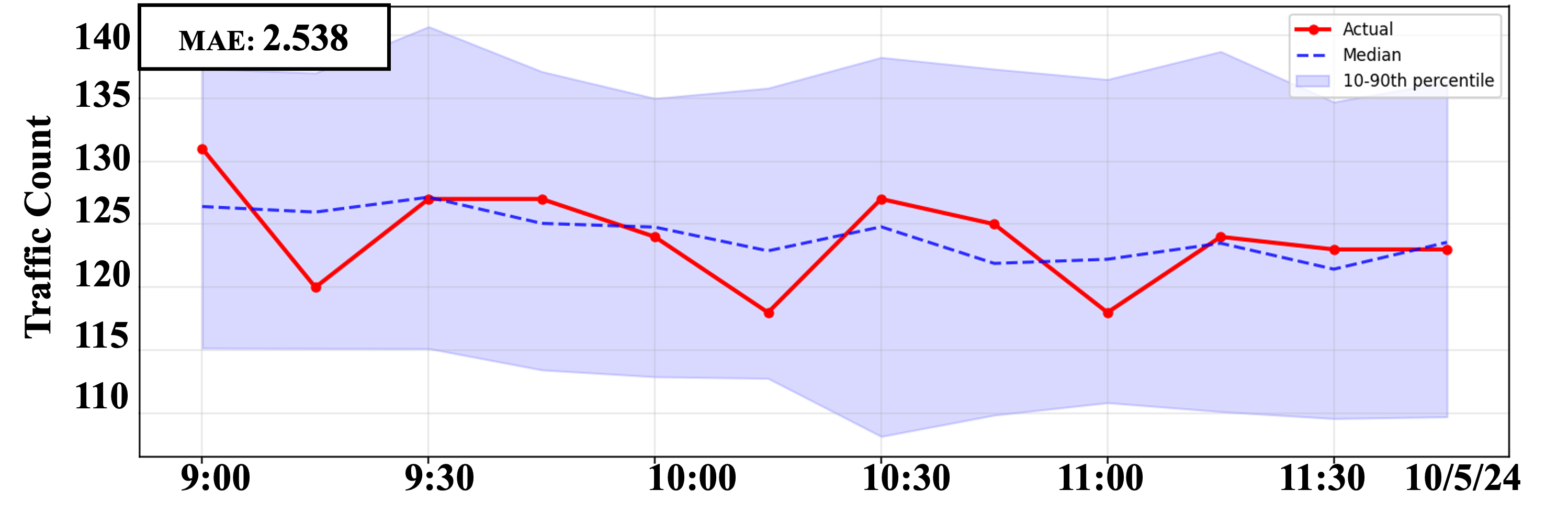}
        \caption{\small Teger (ours)}
        \label{fig4_d}
\end{subfigure}
    \caption{Qualitative comparison of results forecasted by baseline and variants of Teger for the Brussels dataset, annotated by the corresponding MAE.}
    \label{fig:qualitative_and_curvature}
\end{figure}

\newpage
\section*{NeurIPS Paper Checklist}

\begin{enumerate}

\item {\bf Claims}
    \item[] Question: Do the main claims made in the abstract and introduction accurately reflect the paper's contributions and scope?
    \item[] Answer: \answerYes{} 
    \item[] Justification: We clarify the contributions and the scope in the abstract and the introduction.

\item {\bf Limitations}
    \item[] Question: Does the paper discuss the limitations of the work performed by the authors?
    \item[] Answer: \answerYes{} 
    \item[] Justification: It has been discussed in Section \ref{sec:limitations}.

\item {\bf Theory assumptions and proofs}
    \item[] Question: For each theoretical result, does the paper provide the full set of assumptions and a complete (and correct) proof?
    \item[] Answer: \answerYes{} 
    \item[] Justification: Theorems and Lemmas are presented in the main texts. Proofs are explained in the Appendix~\ref{app:proofs}.

    \item {\bf Experimental result reproducibility}
    \item[] Question: Does the paper fully disclose all the information needed to reproduce the main experimental results of the paper to the extent that it affects the main claims and/or conclusions of the paper (regardless of whether the code and data are provided or not)?
    \item[] Answer: \answerYes{} 
    \item[] Justification: We provide details in Section~\ref{sec:experiments}.

\item {\bf Open access to data and code}
    \item[] Question: Does the paper provide open access to the data and code, with sufficient instructions to faithfully reproduce the main experimental results, as described in supplemental material?
    \item[] Answer: \answerYes{} 
    \item[] Justification: We provide the data and code in supplementary materials. We are not allowed to publish the dataset for Brussels traffic because it is private. However, the scripts to reproduce the results are uploaded.

\item {\bf Experimental setting/details}
    \item[] Question: Does the paper specify all the training and test details (e.g., data splits, hyperparameters, how they were chosen, type of optimizer) necessary to understand the results?
    \item[] Answer: \answerYes{} 
    \item[] Justification: We provide details in Section~\ref{sec:experiments}.

\item {\bf Experiment statistical significance}
    \item[] Question: Does the paper report error bars suitably and correctly defined or other appropriate information about the statistical significance of the experiments?
    \item[] Answer: \answerYes{} 
    \item[] Justification: We provide details in Section~\ref{sec:experiments}.

\item {\bf Experiments compute resources}
    \item[] Question: For each experiment, does the paper provide sufficient information on the computer resources (type of compute workers, memory, time of execution) needed to reproduce the experiments?
    \item[] Answer: \answerYes{} 
    \item[] Justification: We provide details in Section~\ref{sec:experiments}.
    
\item {\bf Code of ethics}
    \item[] Question: Does the research conducted in the paper conform, in every respect, with the NeurIPS Code of Ethics \url{https://neurips.cc/public/EthicsGuidelines}?
    \item[] Answer: \answerYes{} 
    \item[] Justification: This paper conforms, in every respect, with the NeurIPS Code of Ethics.

\item {\bf Broader impacts}
    \item[] Question: Does the paper discuss both potential positive societal impacts and negative societal impacts of the work performed?
    \item[] Answer: \answerNA{} 
    \item[] Justification: We think our work will not have a significant social impact.
    
\item {\bf Safeguards}
    \item[] Question: Does the paper describe safeguards that have been put in place for responsible release of data or models that have a high risk for misuse (e.g., pre-trained language models, image generators, or scraped datasets)?
    \item[] Answer: \answerNA{} 
    \item[] Justification: The paper poses no such risks.

\item {\bf Licenses for existing assets}
    \item[] Question: Are the creators or original owners of assets (e.g., code, data, models), used in the paper, properly credited and are the license and terms of use explicitly mentioned and properly respected?
    \item[] Answer: \answerYes{} 
    \item[] Justification: All assets used in this paper are properly credited.

\item {\bf New assets}
    \item[] Question: Are new assets introduced in the paper well documented and is the documentation provided alongside the assets?
    \item[] Answer: \answerYes{} 
    \item[] Justification: We provide the documentation in the code repo.

\item {\bf Crowdsourcing and research with human subjects}
    \item[] Question: For crowdsourcing experiments and research with human subjects, does the paper include the full text of instructions given to participants and screenshots, if applicable, as well as details about compensation (if any)? 
    \item[] Answer: \answerNA{} 
    \item[] Justification: The paper does not involve crowdsourcing nor research with human subjects.

\item {\bf Institutional review board (IRB) approvals or equivalent for research with human subjects}
    \item[] Question: Does the paper describe potential risks incurred by study participants, whether such risks were disclosed to the subjects, and whether Institutional Review Board (IRB) approvals (or an equivalent approval/review based on the requirements of your country or institution) were obtained?
    \item[] Answer: \answerNA{} 
    \item[] Justification: The paper does not involve crowdsourcing nor research with human subjects.

\item {\bf Declaration of LLM usage}
    \item[] Question: Does the paper describe the usage of LLMs if it is an important, original, or non-standard component of the core methods in this research? Note that if the LLM is used only for writing, editing, or formatting purposes and does \emph{not} impact the core methodology, scientific rigor, or originality of the research, declaration is not required.
    \item[] Answer: \answerNA{} 
    \item[] Justification: The core method development in this research does not involve LLMs as any important, original, or non-standard components.

\end{enumerate}

\end{document}